\documentclass[letterpaper, 10 pt, conference]{ieeeconf}  

\IEEEoverridecommandlockouts                              

\title{\LARGE \bf
DroneDiffusion: Robust Quadrotor Dynamics Learning with Diffusion Models
}

\author{Avirup Das$^{*1}$, Rishabh Dev Yadav$^{*1}$, Sihao Sun$^{2}$, Mingfei Sun$^{1}$, Samuel Kaski$^{1,3}$, Wei Pan$^{1}$
\thanks{$^{1}$ Department of Computer Science, The University of Manchester, United Kingdom {\tt \footnotesize (\{avirup.das, rishabh.yadav\}@postgrad.manchester.ac.uk, \{mingfei.sun,\hspace{-0.5mm}samuel.kaski,\hspace{-0.75mm}wei.pan\}@manchester.ac.uk}).  }
\thanks{$^{2}$ Department of Cognitive Robotics, Delft University of Technology, Netherlands ({\tt \footnotesize s.sun-2@tudelft.nl).} }
\thanks{$^{3}$ Department of Computer Science, Aalto University, Finland. }
\thanks{$^{*}$ Equal contribution. }
}
\usepackage{amsmath}
\usepackage{amssymb}
\usepackage{mathtools}
\usepackage{algorithm, algpseudocode}
\usepackage{graphicx, wrapfig, xcolor, multirow, cite}
\usepackage{url}
\usepackage{bbm}
\newcommand{\bmx}[1]{\begin{bmatrix}#1\end{bmatrix}}

\newcommand{\norm}[1]{\left\lVert#1\right\rVert}

\newcommand{\fbr}[1]{\left(#1\right)}
\newcommand{\tbr}[1]{\left[#1\right]}
\newcommand{\sbr}[1]{\left\{#1\right\}}

\newcommand{\expectation}[2]{\mathop{\mathbb{E}}_{#2}\tbr{#1}}

\newtheorem{theorem}{Theorem}[section]

\newtheorem{assumption}[theorem]{Assumption}
\newtheorem{remark}{Remark}

\newcommand{\bz}{\mathbf{z}}

\DeclareMathOperator{\diag}{diag}

\usepackage{array}
\algnewcommand{\Initialise}[1]{%
  \State \textbf{Initialise:}
  \Statex \hspace*{\algorithmicindent}\parbox[t]{.8\linewidth}{\raggedright #1}
}

\begin{document}

\maketitle
\thispagestyle{empty}
\pagestyle{empty}

\begin{abstract}
    An inherent fragility of quadrotor systems stems from model inaccuracies and external disturbances. These factors hinder performance and compromise the stability of the system, making precise control challenging. Existing model-based approaches either make deterministic assumptions, utilize Gaussian-based representations of uncertainty, or rely on nominal models, all of which often fall short in capturing the complex, \emph{multimodal} nature of real-world dynamics. This work introduces \emph{DroneDiffusion}, a novel framework that leverages conditional diffusion models to learn quadrotor dynamics, formulated as a sequence generation task. DroneDiffusion achieves superior generalization to unseen, complex scenarios by capturing the temporal nature of uncertainties and mitigating error propagation. We integrate the learned dynamics with an adaptive controller for trajectory tracking with stability guarantees. Extensive experiments in both simulation and real-world flights demonstrate the robustness of the framework across a range of scenarios, including unfamiliar flight paths and varying payloads, velocities, and wind disturbances. \\
    Project page: \url{https://sites.google.com/view/dronediffusion}
\end{abstract}

\section{INTRODUCTION}

Robust and reliable control of quadrotors is crucial for their expanding role in real-world applications, ranging from autonomous inspection to agile maneuvering~\cite{yao2024autonomous, idrissi2022review}. Achieving this requires designing accurate dynamic models that facilitate precise trajectory tracking and maneuver execution~\cite{torrente2021data}.  However, the development of such models is hindered by the complex interplay of aerodynamic effects, unmodeled dynamics, and external disturbances. Despite the use of uncertainty bounds and adaptation laws, Model-Based control approaches~\cite{labbadi2020robust, xie2021adaptive} remain difficult to implement due to the poorly characterized and stochastic uncertainties inherent in real-world quadrotor environments.

Recent efforts have leveraged Gaussian Processes (GPs) to model aerodynamic effects~\cite{torrente2021data, crocetti2023gapt}, unknown disturbances~\cite{schmid2022real}, and unmodeled dynamics~\cite{wang2018safe}. While these methods provide a convenient data-driven strategy, the nonparametric nature of the GP-based model leads to high computational complexity, which increases with the dataset size, requiring the careful selection of representative training points~\cite{schmid2022real}. As an alternative, Deep Neural Networks (DNNs) have been explored for system identification~\cite{punjani2015deep, bansal2016learning, li2017deep}, with improvements from Spectrally Normalized DNNs~\cite{shi2019neural}. However, despite their advantages, DNNs establish a deterministic mapping from the sensory inputs to the uncertain dynamics, limiting their adaptability to real-world flight conditions. 
\begin{figure}
    \centering
\includegraphics[width=0.96\linewidth]{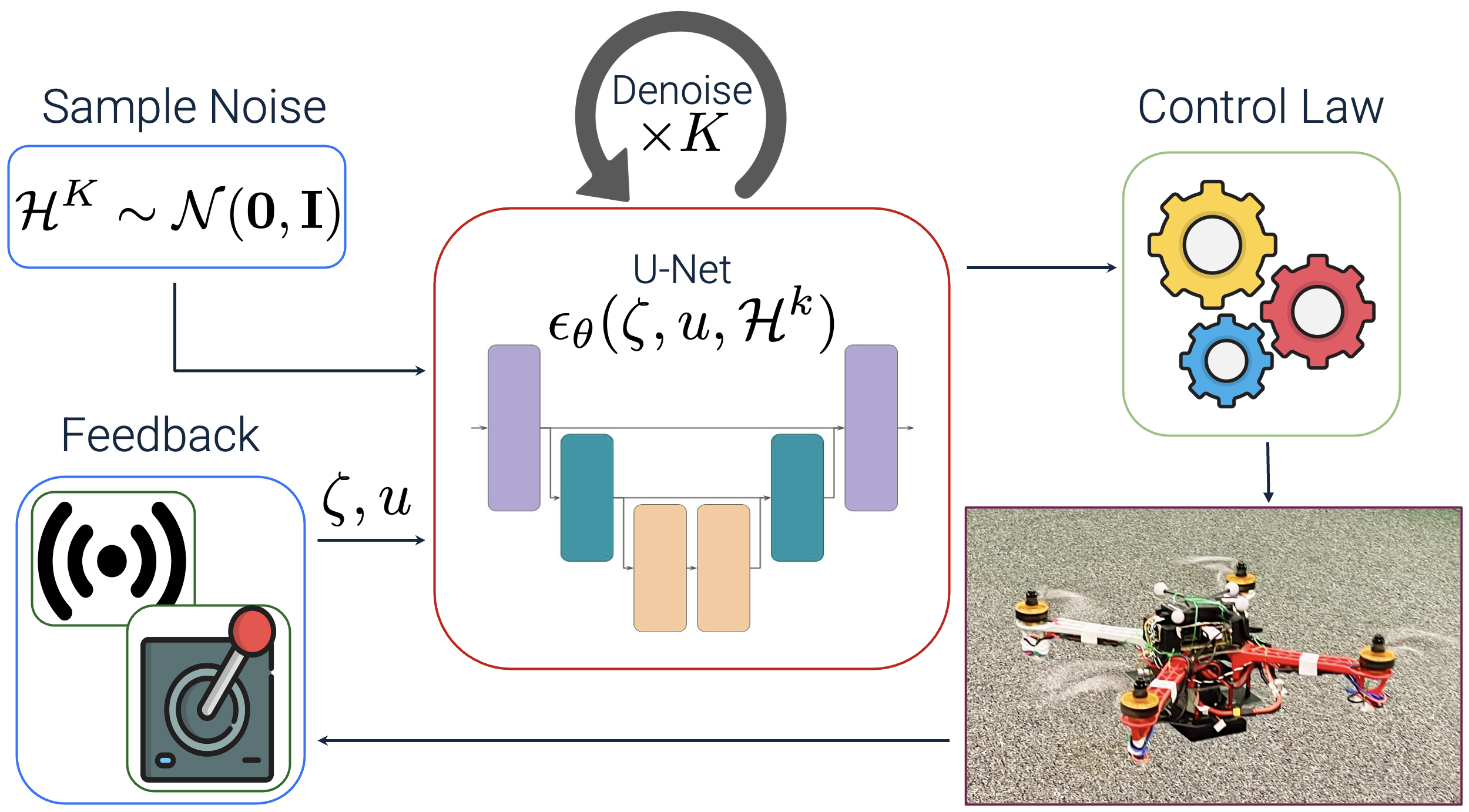}
    \caption{
DroneDiffusion leverages a diffusion model trained on sequences of sensory (state) data ($\zeta$), control inputs ($u$), and corresponding unified residual dynamics terms $\mathcal{H}$ to learn the conditional distribution $p(\mathcal{H} | \zeta, u)$. During flight, the framework uses a U-Net based noise predictor $\epsilon_\theta$ to denoise $\mathcal{H}^K$ (sampled from standard Gaussian $\mathcal{N}(\mathbf{0},\mathbf{I})$) in $K$ steps using $(\zeta, u)$ to generate a prediction of $\hat{\mathcal{H}}$. $\hat{\mathcal{H}}$ is then used by a nonlinear controller in a receding-horizon fashion for closed-loop feedback control of the quadrotor.}
\label{fig:diagram3}
\vspace{-6mm}
\end{figure}
Sequence models like Recurrent Neural Networks (RNNs) have also been effective in modeling dynamics by capturing long-range dependencies through multistep predictions~\cite{mohajerin2018deep, mohajerin2019multistep}. However, they require proper state initialization to avoid vanishing gradients. Transformers (e.g. GPT~\cite{GPT2}) offer improvements over RNNs in capturing long-range dependencies and overall robustness but remain autoregressive in nature which poses challenges in dynamics modeling. Autoregressive predictions can lead to cascading errors, where inaccuracies in initial predictions propagate through the sequence, causing progressively larger deviations over time~\cite{venkatraman2015improving,janner2022diffuser}.

In real-world applications, the uncertainty in quadrotor dynamics is often governed by latent and difficult-to-measure factors such as airflow and downwash~\cite{sun2019quadrotor, Bauersfeld2021NeuroBEMHA, bauersfeld2024robotics}. Developing precise nominal models to capture system dynamics presents significant challenges, necessitating a comprehensive knowledge of system parameters~\cite{saviolo2023active, saviolo2023learning} and limiting the ability of these methods to adapt to varying environmental conditions. Moreover, existing data-driven approaches~\cite{torrente2021data, crocetti2023gapt, wang2018safe, schmid2022real} often rely on deterministic models or assume predefined Gaussian structures for uncertainty, which limits their capacity to capture the full complexity of the dynamics. Additionally, these methods frequently rely on partial uncertainty estimates derived from prior system dynamics information~\cite{mohajerin2018deep, shi2019neural},  further constraining their adaptability to real-world scenarios. A toy experiment with a one-dimensional system and heavy-tailed disturbance (Figure~\ref{fig.toy}) illustrates how standard methods struggle to capture the full distribution of dynamics, highlighting the need for robust models that do not rely on strong assumptions about the underlying dynamics and structure of disturbances.

\begin{figure}[h]
    \centering
    \includegraphics[width=0.99\linewidth]{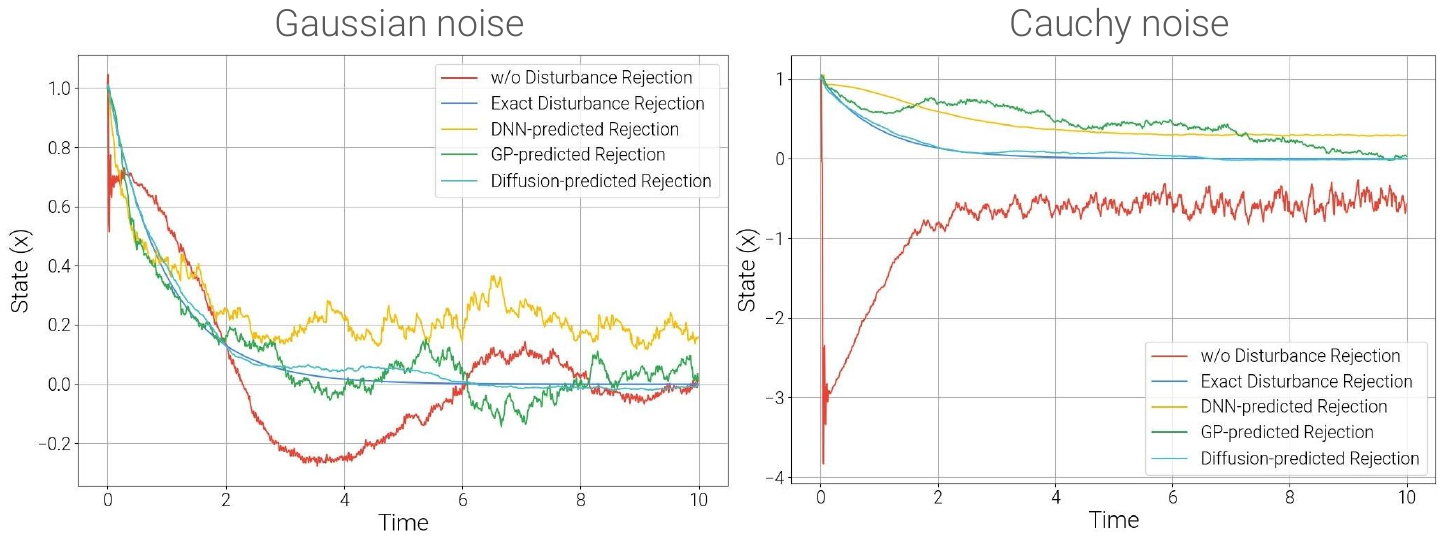}
    \caption{Convergence comparison of the system $\dot{x}(t) = ax(t) + u + d$ with $8$s of training data at $100$Hz. Control inputs are generated by: $u=-Kx- \hat{d}$, where $a = 1$ and $K = 2$. \textit{Left:} $d$ is Gaussian.  \textit{Right:} $d$ is a Cauchy disturbance. GP performs well in the presence of Gaussian disturbance but struggles with heavy-tailed disturbances due to kernel limitations. DNN converges with an offset for heavy-tailed disturbances but struggles to converge effectively when faced with Gaussian disturbance.}
    \label{fig.toy}
\vspace{-2mm}
\end{figure}

To address these limitations, we propose using diffusion models~\cite{sohl2015deep, ho2020denoising, song2020score} for quadrotor dynamics modeling, which have recently demonstrated exceptional performance in fields such as image generation~\cite{dhariwal2021diffusion} and offline reinforcement learning~\cite{janner2022diffuser, ajay2023is}. In this work, we apply diffusion models for the first time to model the complex, multimodal distributions present in quadrotor flight dynamics. Leveraging their ability to capture intricate data distributions and generate high-fidelity samples, we introduce \textbf{DroneDiffusion}, a framework that models quadrotor dynamics as a \textit{conditional probabilistic distribution} and integrates it into a \textit{nonlinear control framework}.
The key contributions of this paper can be summarized as follows:
\begin{itemize}
\item We formulate dynamic model learning as a conditional sequence generation task using diffusion models, accurately characterizing the multimodal and stochastic nature of quadrotor dynamics in the real world. 

\item Our approach integrates the learned dynamic model with adaptive control for trajectory tracking, ensuring stability and adaptability to varying masses and velocities without extensive retraining.
\end{itemize}
Comprehensive real-time experiments validate the robustness of our framework in the face of model uncertainties, demonstrating its ability to accurately track complex trajectories, adapt to dynamic variations, and maintain resilience against wind disturbances, achieving significant advancements over existing solutions.

\section{System Identification with Diffusion Models}
\subsection{Preliminaries on diffusion models}
Denoising Diffusion Probabilistic Models (DDPMs)~\cite{ho2020denoising} operate through forward and reverse processes. The forward process, $q(\bz^k|\bz^{k-1}) \coloneqq \mathcal{N}(\bz^k|\sqrt{1-\beta^k}\bz^{k-1}, \beta^k\mathbf{I})$, progressively transforms the data distribution $q(\bz^0)$ into a standard Gaussian $\mathcal{N}(\mathbf{0}, \mathbf{I})$ over $K$ steps, using a predefined variance schedule $\{\beta^k\}_{k=1}^K$. The reverse process, $p_{\theta}(\bz^{k-1}|\bz^k) \coloneqq \mathcal{N}(\bz^{k-1}|\mu_\theta(\bz^k, k), \Sigma^k)$, parameterized by $\theta$, reconstructs the data by estimating the mean $\mu_\theta$. Training optimizes a variational lower bound on the likelihood $\log p_\theta$ via the following objective: \begin{equation}\label{eqn:ho_simple} \mathcal{J}_{\rm denoise}(\theta)\coloneqq \expectation{\norm{\epsilon - \epsilon_\theta(\bz^k, k)}^2}{k\sim\mathcal{U}(1, K), \bz^0 \sim q, \epsilon \sim \mathcal{N}(\mathbf{0}, \mathbf{I})}, \end{equation} where $\epsilon_\theta(\bz^k, k)$, a neural network parameterized by $\theta$, predicts the noise $\epsilon$ injected during the forward process, enabling the recovery of $\bz^{k-1}$. The reverse process can be conditioned on additional variables $c$ by modifying $\mu_\theta$ to incorporate $c$, resulting in a conditional reverse process: $p_\theta(\bz^{k-1}|\bz^k, c) \coloneqq \mathcal{N}(\bz^{k-1}|\mu_\theta(\bz^k, c, k), \Sigma^k)$.

While DDPMs are commonly applied to image generation, we adapt them to model quadrotor dynamics, inspired by Diffusion Policy~\cite{chi2023diffusionpolicy}. Our formulation introduces two key modifications: (1) redefining the output to represent quadrotor residual dynamics, and (2) conditioning the denoising process on state observations $\boldsymbol{\zeta}_t$ and past control inputs $\mathbf{u}_t$. These changes are discussed below, with an overview provided in Figure~\ref{fig:diagram2}.
\begin{figure}[t]
    \centering
    \includegraphics[width=1\linewidth]{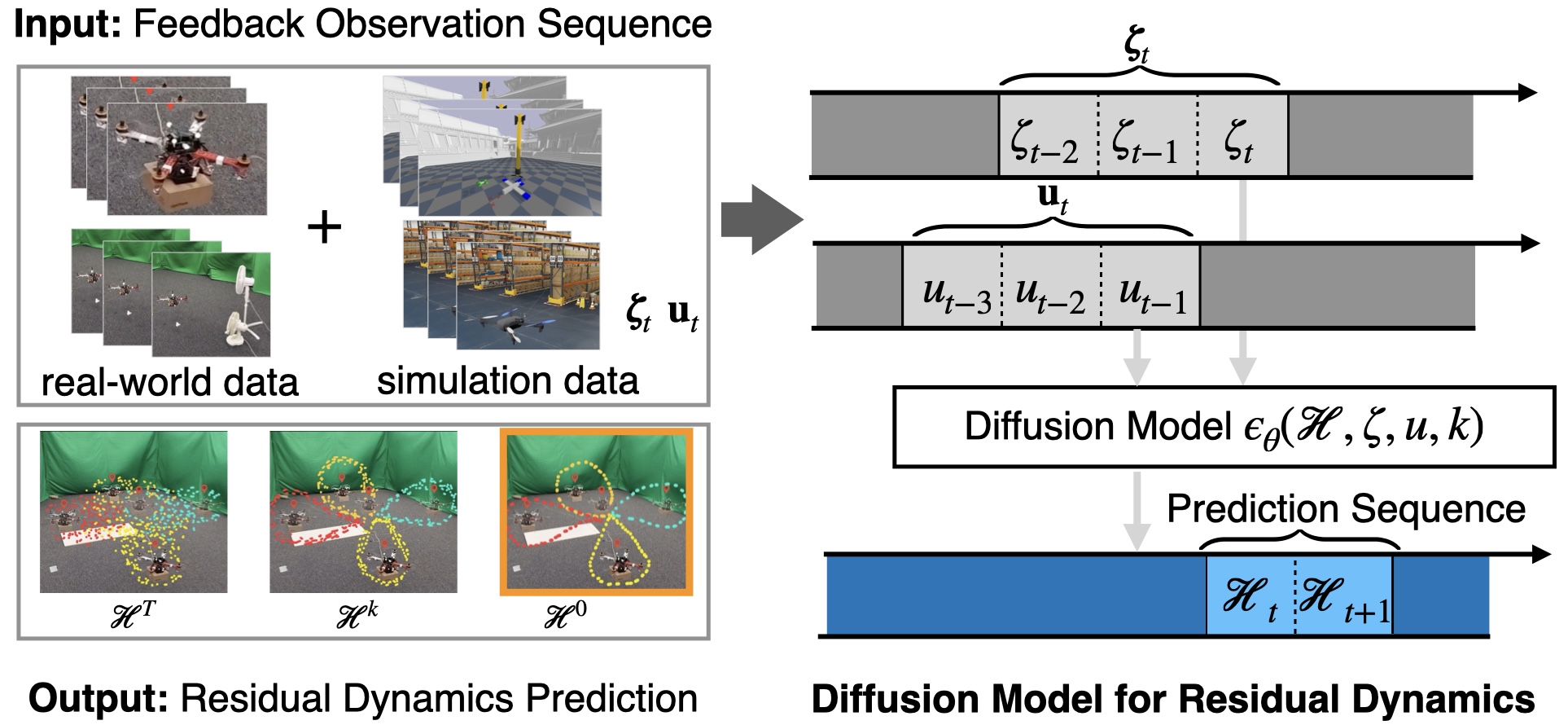}
     \caption{Diffusion Model for Dynamics Learning. A general formulation (see Section~\ref{sec:learningsequence}). At time step $t$, the residual dynamics takes the latest state and control input observation data $\boldsymbol{\zeta}_t$ and $\mathbf{u}_t$ as input and outputs a sequence of residual dynamics prediction $[\boldsymbol{\mathcal{H}}_t,\boldsymbol{\mathcal{H}}_{t+1},\cdots ]$.}
\label{fig:diagram2}
\vspace{-2mm}
\end{figure}

\subsection{Dynamic Model Learning as Sequence Generation}\label{sec:learningsequence}
The position dynamics of a quadrotor~\cite{shi2019neural} are governed by the states: position $p \triangleq \left[p_x,, p_y,, p_z \right]^{\top} \in \mathbb{R}^3$, velocity $v \in \mathbb{R}^3$, and attitude rotation matrix $R \in \mathrm{SO}(3)$ as follows:
\begin{align} 
\dot{p} &= v, & m\dot{v} &= m g_v + R f_u + f_a, \label{eq.quadrotor position_intro} 
\end{align} where $m$ is the mass, $g_v = [0, 0, -g]^\top$ is the gravity vector, $f_u = [0, 0, T]^\top$ represents thrust, and $f_a$ encapsulates unknown aerodynamic forces. To model the uncertain dynamics without assuming  prior structural knowledge of $f_a$, we introduce a nominal mass $\bar{m} \in \mathbb{R}^{+}$, reformulating \eqref{eq.quadrotor position_intro} as: 
\begin{align} 
\bar{m} \ddot{p} + \mathcal{H}(p, \dot{p}, \ddot{p}) &= u, \label{eq.p_tau_rearrage} 
\end{align} 
where $\mathcal{H}(p, \dot{p}, \ddot{p}) \triangleq (m - \bar{m}) \ddot{p} - m g_v - f_a$, and $u = R f_u$. The term $\mathcal{H}$ captures state-dependent uncertainties arising from aerodynamic effects $f_a$ (e.g., drag, ground effect, wind) and imperfect information of mass $m$ (e.g., payloads),  into a single variable. This reformulation facilitates a data-driven approach to infer the distributional properties of $\mathcal{H}$, which is inherently nonlinear and complex. Unlike single-step dynamic models, which often lack the temporal context required to infer the underlying causal structure of the dynamics and are prone to compounding errors, our approach frames the estimation of $\mathcal{H}$ as a sequence generation task. This allows us to mitigate error propagation and make temporally consistent predictions. 

For a time-horizon $H$, we define a sequence $\mathcal{S}$ comprising uncertainties $\mathcal{H}$, control inputs $u$, and sensory measurements $\zeta = \{p, \dot{p}, \ddot{p}\}$: \begin{equation}
\label{eqn:sequence} \mathcal{S} = \bmx{\zeta_1 & \zeta_2 & \ldots & \zeta_H\\
u_1 & u_2 & \ldots & u_H\\
\mathcal{H}_1 & \mathcal{H}_2& \ldots & \mathcal{H}_H}. 
\end{equation} This induces a tight coupling between $\mathcal{H}$ and $(\zeta,u)$ that makes the estimation of $\mathcal{H}$ equivalent to sampling from the \emph{conditional distribution} $p(\mathcal{H}\vert \zeta,u)$. By predicting $\mathcal{H}$ over a horizon, we achieve temporally consistent predictions, ensuring a stable flight. During flight, the control input at the current timestep $u_t$, is unavailable for estimating $\mathcal{H}$, so we use the control input from the previous timestep $u_{t-1}$, assuming one-step difference of control signal are small and bounded \cite{shi2019neural}. For simplicity, we retain the original notation in \eqref{eqn:sequence}. At the first timestep, the previous control input is taken as zero.

\subsection{Diffusion Model for Multimodal Dynamics}

\begin{figure}[!h]
\centering
\includegraphics[width=0.95\linewidth]{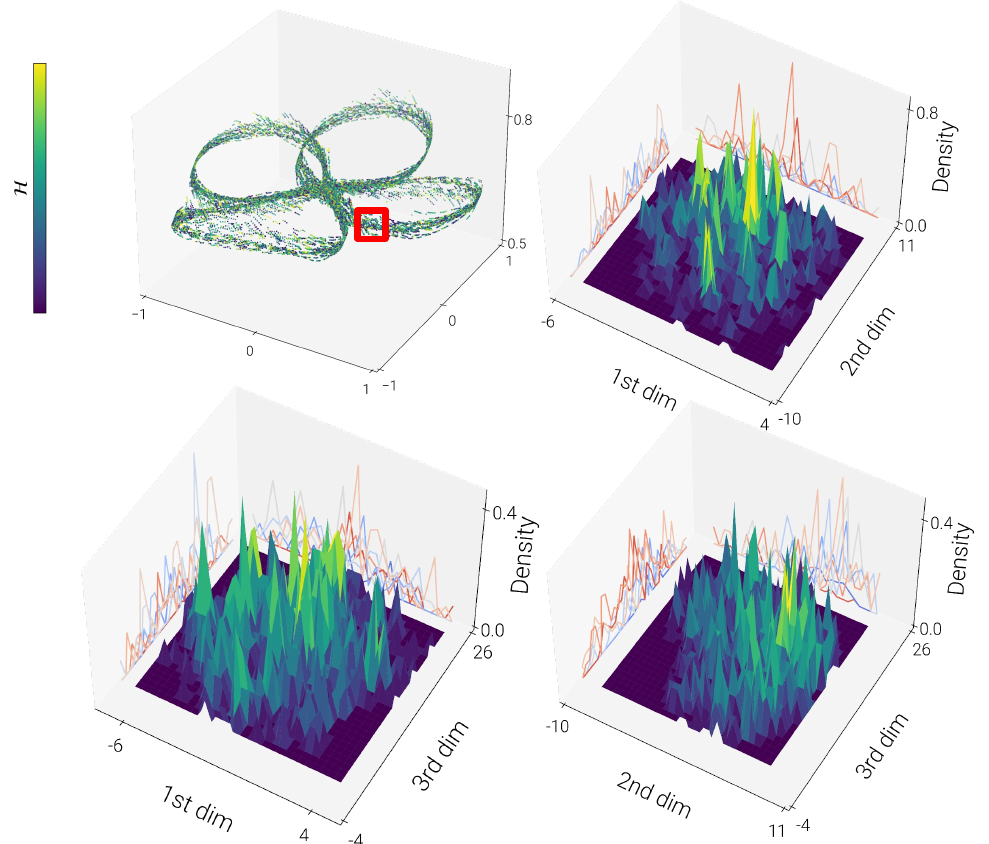}
    \caption{\textit{Top left:} Position of the quadrotor during real-world flights with the uncertainties at each timestep $\mathcal{H}_t$ highlighted. The following plots illustrate the empirical density of $\mathcal{H}$ observed in the {\color{red}red box} along two dimensions. \textit{Top right:} Density along first and second dimension, \textit{Bottom left:} Density along first and third dimension, and \textit{Bottom right:} Density along second and third dimension. Even under ideal conditions with no external disturbances,  $\mathcal{H}$ observed in a given flight path segment (expected to have similar control inputs and sensor feedback) exhibits a multimodal nature.
    }
    \label{fig:multimodal}
\end{figure}
Quadrotor dynamics is inherently multimodal, driven by a combination of environmental and mechanical factors that introduce significant uncertainties. Turbulent wind gusts, temperature-induced variations in air density, and aerodynamic effects such as propeller-induced flows and ground effects result in nonlinearities. Mechanical factors, including motor degradation, IMU sensor noise, and control delays, further complicate the structure of $\mathcal{H}$. Many of these parameters—such as air density, humidity, and downwash effects—are difficult to measure with onboard sensors, and their variations due to altitude and temperature fluctuations amplify uncertainties during flight. Repeated flights along an infinity-shaped trajectory with a fixed PID controller (Figure \ref{fig:multimodal}, top left) reveal varying levels of $\mathcal{H}$ even at identical flight segments. The remaining panels in Figure \ref{fig:multimodal} confirm the multimodal structure of $\mathcal{H}$ in real-world operating regimes. The empirical density of $\mathcal{H}$ for a specific segment of the trajectory (highlighted in {\color{red}red}) reveals distinct modes, indicating variability in the dynamics even without external disturbances~(Refer to Appendix A] for quantitative analysis and further discussion). 

\begin{figure}[!h]
    \centering
    \includegraphics[width=0.95\linewidth]{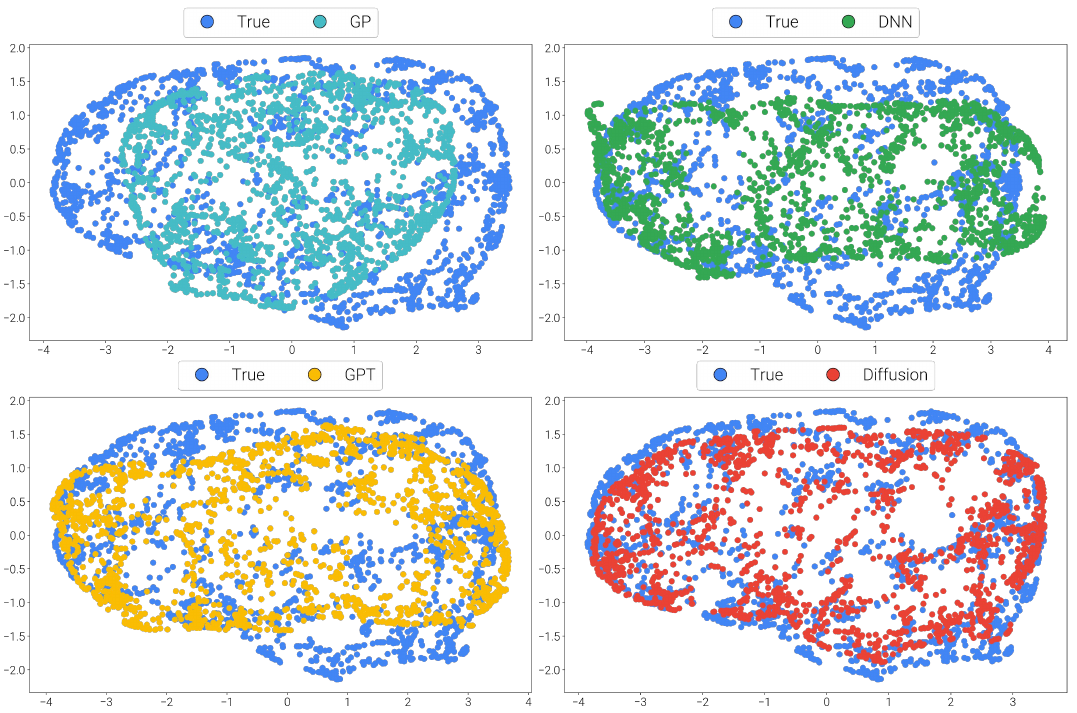}
    \caption{t-SNE plots of $\mathcal{H}$ observed from real-world flight data ({\color{red}red box} in Figure \ref{fig:multimodal}) and $\hat{\mathcal{H}}$ obtained from the baselines and the proposed Diffusion model. While GPT fairly estimates the uncertainties in the dynamics, the proposed diffusion model accurately captures the underlying distribution.}
    \label{fig:tsne}
\end{figure}

If $\mathcal{H}_t$ is modeled as independent multimodal distributions, such as in Gaussian Mixture Models, successive predictions may originate from different modes, leading to temporal inconsistencies in the dynamics~\cite{chi2023diffusionpolicy}. To address this issue, we leverage the ability of DDPMs to capture the distribution of $\mathcal{H}$. Unlike traditional methods, DDPMs leverage Stochastic Langevin Dynamics~\cite{welling2011bayesian} in the denoising process, allowing them to capture multimodal distributions that naturally reflect the inherent stochasticity of the system. Building on the success of observation-conditioned visuomotor policies~\cite{chi2023diffusionpolicy}, we use the DDPM to approximate the conditional distribution $p(\mathcal{H} \vert \zeta, u)$, instead of the joint distribution $p(\mathcal{H}, \zeta, u)$ typically done in planning~\cite{janner2022diffuser}. This allows us to better capture the uncertainties by conditioning on sensory inputs $\zeta$ and control inputs $u$. To achieve this, we modify the reverse process of DDPM as:
\begin{equation}\label{eq:denoising}
    \mathcal{H}^{k-1}= \frac{1}{\sqrt{\alpha^k}}\fbr{\mathcal{H}^k-\frac{1-\alpha^k}{\sqrt{1-\bar{\alpha}^k}}\epsilon_\theta(\zeta,u,\mathcal{H}^k,k)+\sqrt{\beta^k}\mathbf{z}},
\end{equation}
where $\alpha^k\coloneqq1-\beta^k, \bar{\alpha}^k\coloneqq\prod_{i=1}^k\alpha^i$ and $\mathbf{z}\sim\mathcal{N}(\mathbf{0},\mathbf{I})$. With an offline dataset $\mathcal{D}=\sbr{\mathcal{S}}_{i=1}^M$ and the denoising process \eqref{eq:denoising}, the simplified objective \eqref{eqn:ho_simple} becomes:
\begin{equation}\label{eqn:dynamics_objective}
    \mathcal{J}(\theta)= \expectation{\norm{\epsilon-\epsilon_\theta(\zeta,u,\mathcal{H}^k,k)}^2}{\substack{k\sim\mathcal{U}(1,K), \epsilon\sim\mathcal{N}(\mathbf{0},\mathbf{I})\\
    \fbr{\zeta,u,\mathcal{H}^0}\sim\mathcal{D}}}.
\end{equation}
We now discuss the deployment of the trained diffusion model in real-world flight scenarios. The process begins by specifying the desired waypoints $\zeta_d$ and collecting real-time sensory data $\zeta$ for trajectory tracking. Using the trained model, we generate a conditional sample $\hat{\mathcal{H}}$ via \eqref{eq:denoising}, where the covariance follows the cosine schedule proposed in \cite{nichol2021improved}. The uncertainty for the first timestep $\mathcal{\hat{H}}_1$ from the resulting estimate is then utilized to compute the control inputs $u$. This procedure is iteratively applied within a receding-horizon control framework, as depicted in Figure \ref{fig:diagram3}, ensuring ongoing adaptation to the dynamic flight conditions. Figure \ref{fig:tsne} illustrates the effectiveness of the proposed diffusion model in capturing the intricate structure of $\mathcal{H}$.

\subsection{Diffusion Model-based Adaptive Control (DM-AC)}
\label{sec.control_design}
Given the estimate $\hat{\mathcal{H}}$ obtained from the proposed diffusion model, we now proceed to design a closed-loop feedback controller to ensure that the position tracking error $e_p(t) \triangleq p(t) - p_d (t)$ converges to zero, ideally at an exponential rate. This can be achieved through the design of a sliding variable $s$, which defines a manifold that guarantees the desired convergence rate of the tracking error:
\begin{align}
     s(t) &= \dot{e}_p(t) + \Phi e_p(t), 
     \label{sliding_var}
 \end{align}
where $\Phi$ is a user-defined positive definite gain matrix. 
The adaptive controller can be derived as:
\begin{subequations}
\begin{align}
    u(t) &=   -\Lambda s + \bar{m} (\ddot{p}_d - \Phi \dot{e}_p) + \hat{\mathcal{H}} -  \hat{\sigma}(t) (s/ \norm{s}),   \label{input_main}\\ 
    \dot{\hat{\sigma}}(t) &= \norm{s} - \nu \hat{\sigma}(t), \quad \hat{\sigma}(0) > 0, \label{adaptive_law_main}
\end{align}
\label{control_law_adaptive_main}
\end{subequations} 
where $\Lambda$ is a positive definite user-defined gain matrix, $\nu \in  \mathbb R^{+} $ are user-defined scalars and $\hat{\sigma}(t)$ is adaptive gain. This not only ensures the dynamic compensation through $\hat{\mathcal{H}}$ but also accounts for estimation errors in $\hat{\mathcal{H}}$ resulting from unknown external disturbances. Theorem \ref{theorem:adaptive_control} ensures the closed-loop stability of the system using DM-AC (Refer to Appendix B for proof).
\begin{assumption} \label{assum1}
The desired states along the position trajectory $p_d$, $\dot{p}_d$, and $\ddot{p}_d$ are bounded. Also, the external disturbances are bounded.
\end{assumption}

\begin{assumption} \label{assum2}
    The learning error of $\hat{\mathcal{H}}$ over the
compact sets $\zeta \in \mathcal{Z}$, $u \in \mathcal{U}$ is upper bounded by $\sigma_m = sup_{\zeta, u} \norm{\sigma }$, where $\sigma = \mathcal{H} - \hat{\mathcal{H}}$.
\end{assumption}
 \begin{theorem} \label{theorem:adaptive_control}
Under Assumptions \ref{assum1}, \ref{assum2} the closed-loop trajectories of \eqref{eq.p_tau_rearrage}  with control law \eqref{input_main} along with the adaptive law \eqref{adaptive_law_main} are Uniformly Ultimately Bounded (UUB) (Definition $4.6$ in \cite{khalil2015nonlinear}).
 \end{theorem}
\begin{remark}
    While Assumption \ref{assum2} is necessary for ensuring system stability (Theorem \ref{theorem:adaptive_control}), it is well-founded and aligns with recent findings for Score-based Generative Models (SGM)~\cite{song2020score}. Works like \cite{lee2022convergence, de2022convergence, chen2023sampling} show that SGMs can achieve convergence to the data distribution by bounding the divergence between the learned and true distributions. This suggests that Assumption \ref{assum2} is satisfied when the training data is sufficiently representative of the underlying distribution of $\mathcal{H}$. Moreover, $\sigma_m$ is unknown for control design of \eqref{control_law_adaptive_main}.
\end{remark}
\begin{algorithm}
  \caption{Quadrotor Trajectory Tracking with DM-AC}\label{algorithm:diffusion_sample_adaptive}
  \begin{algorithmic}[1]
    \Require Noise model $\epsilon_\theta$, covariances $\Sigma^k$,  gains $\bar{m}, \Lambda, \Phi, {\nu}$
    \Statex \hspace{-6.9mm}{ \textbf{Initialize:}} $u=0, {\hat{\sigma}(0)}$
    \While{not done}
    \State Get desired way-point $\zeta_d = \fbr{p_d, \dot{p}_d, \ddot{p}_d}$
    \State Observe sensor feedback $\zeta$, initialize $\mathcal{H}^{\mathrm{T}}\sim\mathcal{N}(\mathbf{0},\mathbf{I})$
    \For{k=T,\ldots,1}
    \State $(\mu^{k-1},\Sigma^k) \gets \mathrm{Denoise}(\mathcal{H}^k, \epsilon_\theta(\zeta, u, \mathcal{H}^k,k))$ 
    \State $\mathcal{H}^{k-1}\sim \mathcal{N}(\mu^{k-1}, (1-\beta)\Sigma^{k-1})$
    \EndFor
    \State Derive $e_p$, $\dot{e}_p$ and calculate $s$ as in \eqref{sliding_var} 
    \State Calculate $\hat{\sigma}$ using \eqref{adaptive_law_main}
    \State Update $u$ using \eqref{input_main}
    \EndWhile
  \end{algorithmic}
\end{algorithm}
\vspace{-2mm}
Algorithm~\ref{algorithm:diffusion_sample_adaptive} details the use of the diffusion sample $\hat{\mathcal{H}}$ with the adaptive controller \eqref{control_law_adaptive_main} for trajectory tracking\footnote{An alternative control law is presented in Appendix C, demonstrating the flexibility of the framework.}. The desired total thrust and the desired attitude can be calculated as per \cite{mellinger2011minimum}. 

\section{Experiments}
The primary objective of our experiments is to rigorously evaluate the efficacy of the proposed methodology in meeting the performance expectations of a data-driven trajectory tracking framework. Specifically, we aim to address the following key questions: \textbf{(1)} How accurately can the framework track complex flight paths that were not encountered during training? \textbf{(2)} How well does the framework generalize to varying dynamic conditions and external disturbances beyond the training distribution? We conclude by studying practical runtime considerations of the framework.

\textbf{Simulation Setup:} We utilize \textit{Iris} quadrotor, deployed on the open-source \textit{PX4-Autopilot} flight control software, integrated within the Gazebo simulation environment.

\textbf{Hardware Setup:} The experimental setup includes an 8-camera OptiTrack motion capture system and a quadrotor (approximately 1.4 kg) featuring DYS D4215 - 650 KV brushless motors, powered by a 4S (120 C) LiPo battery, and equipped with 10-inch propellers on an F-450 frame, as shown in Figure \ref{fig:diagram3} (bottom-right). The quadrotor is equipped with a Pixhawk flight controller running PX4, an onboard Jetson Xavier NX computer running the \textit{MAVROS} package, and a host computer with an NVIDIA GeForce RTX 4080 for inference. Sensor data is transmitted from the onboard computer to the host computer via WiFi, where it is processed, and control inputs are sent back to the onboard computer at a frequency of 30 Hz. The hardware configuration reflects a standard setup typical of consumer drone systems. The PX4 is operated in offboard mode, receiving desired thrust and attitude commands from the position control loop \eqref{control_law_adaptive_main}. The built-in PX4 multicopter attitude controller (linear PID controller based on quaternion error) was executed at the default rate. Desired trajectories were generated using the open-source ROS package `\textit{mav\_trajectory\_generation}'\cite{mav_trajectory_generation}.

\textbf{Implementation:} We construct an offline dataset of sequences $\mathcal{S}$ by computing $\mathcal{H}$ from \eqref{eq.p_tau_rearrage} and train the diffusion model to minimize \eqref{eqn:dynamics_objective}. The noise predictor $\epsilon_\theta$ is parameterized by a temporal U-Net~\cite{janner2022diffuser}, which uses three repeated residual blocks, each with two temporal convolutions, followed by group normalization~\cite{wu2018group} and Mish activation~\cite{Misra2020MishAS}. Timestep and condition embeddings (16-dimensional vectors) are produced by separate 2-layer MLPs with 64 hidden units and Mish activation. We train $\epsilon_\theta$ using the Adam optimizer~\cite{kingma2014adam} with a learning rate of $2 \times 10^{-4}$ and a batch size of 256 for 50,000 steps, using $K = 20$ diffusion timesteps. The adaptive controller gains are set as: $\Phi = \diag\{1.5, 1.5, 1.2\}$, $\Lambda = \diag\{2.0, 2.0, 4.0\}$, with $\bar{m} = 1$, $\hat{\sigma}(0) = 0.1$, and $\nu = 2.0$. 

\textbf{Baselines:} For a baseline comparison, we implement a first-principles-based $\mathcal{L}_1$-Adaptive controller~\cite{wu2023l1}\footnote{https://github.com/sigma-pi/L1Quad}, which is widely recognized for its proven stability and performance in adaptive control tasks. To demonstrate the effectiveness of the proposed diffusion-based framework, we compare it against strong, established data-driven model-learning approaches, including GPs \cite{torrente2021data}\footnote{https://github.com/uzh-rpg/data\_driven\_mpc} and DNNs~\cite{shi2019neural}, which have shown success in learning quadrotor dynamics, and autoregressive models such as GPT-2 \cite{chen2021decision}\footnote{https://github.com/kzl/decision-transformer}, commonly used for sequence prediction tasks. For a fair tracking comparison, all models were trained on the same offline dataset to predict $\hat{\mathcal{H}}$ and were used with the control law  \eqref{control_law_adaptive_main}. Results are reported for 10 individual trials.

\subsection{Generalization to Unseen Complex Trajectories}\label{sec:primitives}

\begin{figure}[!h]
\centering
\includegraphics[scale=0.2]{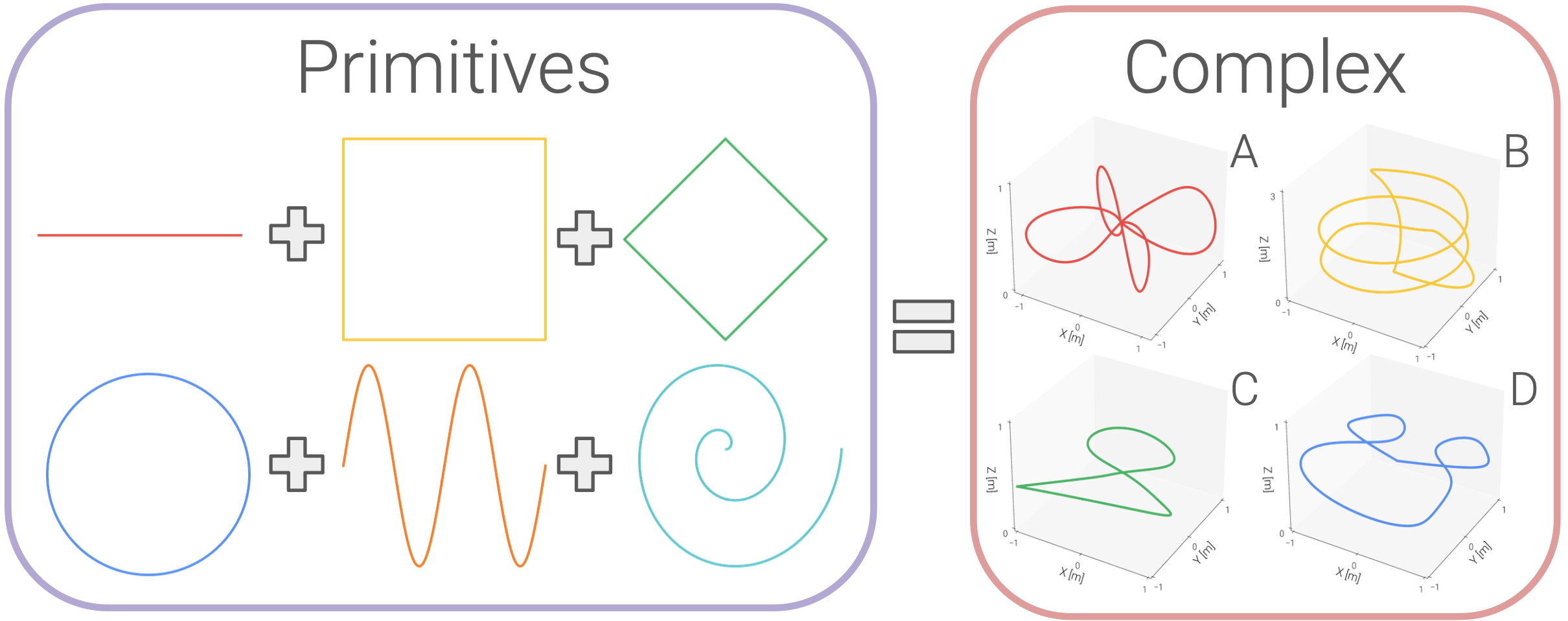}   
    \caption{\textit{Left:} Primitive trajectories, such as straight lines, squares, kites, circles, sinusoidal waves, and spirals, were executed in all three planes ($xy, yz, xz$) to capture a variety of motion patterns. \textit{Right:} Complex trajectories include a combination of sharp turns, loops, and simultaneous motion in all three ($x, y, z$) coordinates, illustrating challenging real-world flight scenarios.}
    \label{fig:primitve_complex}
\vspace{-4mm}
\end{figure}
To evaluate the generalization capability of the proposed framework to previously unseen environments, we trained the models on single demonstrations of basic \emph{primitive trajectories} (Figure~\ref{fig:primitve_complex}) with a mean velocity of $0.4$ m/s and tested on more \textit{complex trajectories} in a simulated environment. Our evaluation focuses on both predictive accuracy and trajectory tracking performance.
\begin{table}[!h]
\footnotesize
\renewcommand{\arraystretch}{1.4}
\caption{Prediction and Tracking comparison in simulation}
\centering
{
\scalebox{0.75}{
\begin{tabular}{|c||c|c|c|c||c|c|c|c|}
\hline
\multirow{2}{*}{} & \multicolumn{4}{c||}{Model Prediction Error (N)}                                   & \multicolumn{4}{c|}{Trajectory Tracking Error (m)}                                           \\ \cline{2-9} 
\textbf{RMSE} &  A        &  B         &  C         &  D   &  A        &  B         &  C         &  D       \\ \hline
GP-AC               & $0.097 $ & $0.108$ & $0.055 $ & $0.073 $ & $0.047 $ & $0.052$ & $0.039 $ & $0.030 $ \\ \hline
DNN-AC             & $0.091$ & $0.101$ & $0.046$ & $0.062$ & $0.032$ & $0.047$ & $0.031$ & $0.025$ \\ \hline
GPT-AC             & $0.088$ & $0.089$ & $0.044$ & $\mathbf{0.048}$  & $0.031$ & $0.038$ & $0.028$ & $\mathbf{0.022}$ \\ \hline
DM-AC            & $\mathbf{0.073}$ & $\mathbf{0.081}$ & $\mathbf{0.041}$ & $0.051$  & $\mathbf{0.022}$ & $\mathbf{0.021}$ & $\mathbf{0.019}$ & $0.023$ \\ \hline
\multicolumn{9}{c}{*\{$\cdot$\}-AC refers to \{$\cdot$\} based adaptive control.}           \end{tabular}
}}
\label{tab:primitive}
\vspace{-3mm}
\end{table}
For predictive accuracy, we tasked the quadrotor to follow the target trajectory using a vanilla PID controller while calculating $\mathcal{H}$ at each timestep from \eqref{eq.p_tau_rearrage}. We compare the predictive capabilities of the four baselines—GP, DNN, GPT, and our proposed diffusion model—using the features $\{\zeta, u\}$ at each timestep. We then assess trajectory tracking performance by implementing the control law \eqref{control_law_adaptive_main} and computing the root mean square (RMS) position errors for each method, as shown in Table \ref{tab:primitive}. The results demonstrate that the diffusion model consistently achieves superior predictive accuracy and tracking performance, exhibiting lower RMS errors compared to the other baselines. While GPT performs comparably to the diffusion model on one of the \textit{complex trajectories} (Complex D), the diffusion model outperforms all baselines across the remaining cases.
\begin{figure}[!h]
    \centering
\includegraphics[width=0.99\linewidth]{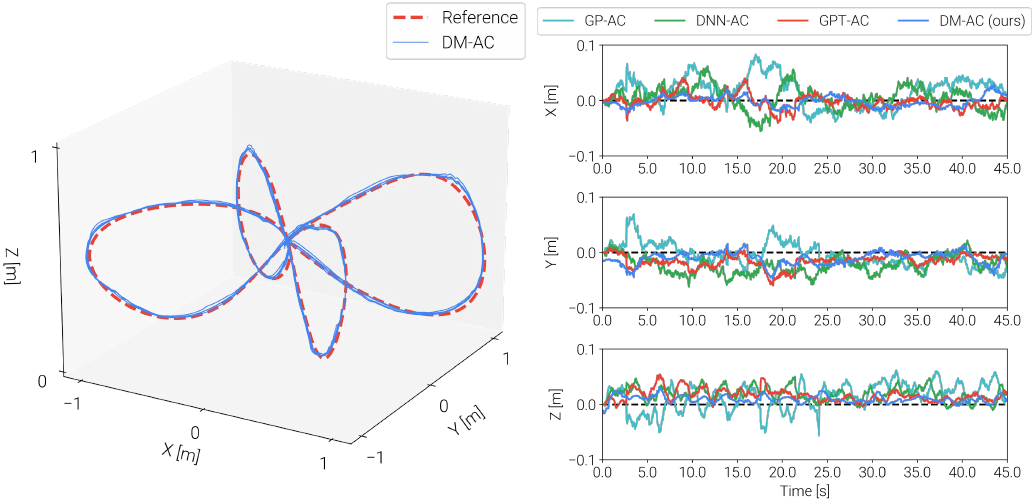}    \caption{\textit{Left:} Comparison between reference and tracked trajectory in real-world flight for DM-AC. \textit{Right:} Tracking error comparison on Complex `A' in real-world flight. RMS tracking errors for GP-AC, DNN-AC, GPT-AC, and DM-AC are $0.075$ m, $0.071$ m, $0.067$ m and $0.054$ m respectively.}
    \label{fig:complex_a}
    \vspace{-2mm}
\end{figure}

Given that aerodynamic effects are more accurately captured in real-world flight than in simulation, we further validate the framework using real-world data. Specifically, we evaluate the performance of DM-AC on the Complex A trajectory (Figure \ref{fig:primitve_complex}), using primitive demonstrations recorded at a velocity of $0.4 m/s$ for training. As shown in Figure \ref{fig:complex_a}, DM-AC achieves the lowest tracking error. This highlights the ability of DM-AC to track complex, unseen trajectories without requiring extensive data collection for each specific trajectory. This flexibility is driven by the learned conditional distribution $p(\mathcal{H}\vert \zeta,u)$, which enables DM-AC to effectively generalize across a wide range of challenging real-world conditions and scenarios.

\subsection{Adaptation against various Payload, Velocities and Wind}
\begin{figure}[!h]
    \centering
\includegraphics[width=0.99\linewidth]{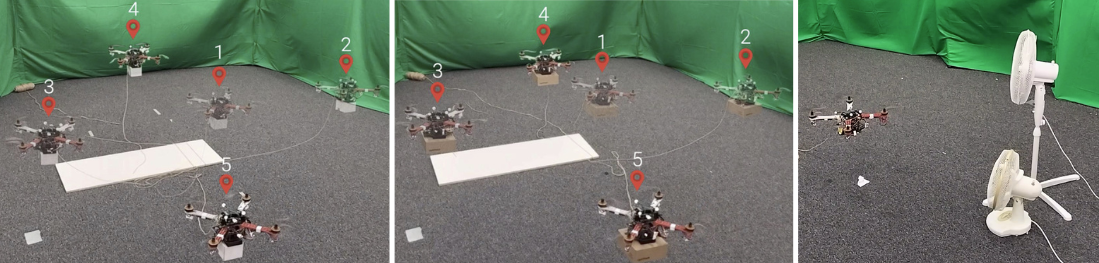}
\caption{\textit{Left} and \textit{Centre:} Sequential snapshots of the quadrotor with payloads $0.4$ kg and $0.8$ kg, respectively. \textit{Right:} Quadrotor in the presence of external wind generated by two fans.}
\label{fig.mass_velocity}
\vspace{-2mm}
\end{figure}

To assess the adaptability of the proposed framework, we test on a complex double infinity-shaped trajectory (where the first major loop is aligned along the $x$-axis and the second along the $y$-axis, as depicted in the top left of Figure \ref{fig:multimodal}) under varying payloads and wind disturbances. The evaluation is conducted under the following conditions: (a) tracking at velocities of $0.35$ m/s and $0.5$ m/s with a payload of $0.4$ kg, (b) tracking at the same velocities with a payload of $0.8$ kg, and (c) hovering at $(0,0,0.8)$ m under external wind disturbances (Figure \ref{fig.mass_velocity}). The diffusion model was trained on \emph{single real-world demonstrations of primitive trajectories} (left Figure \ref{fig.mass_velocity}) with payloads of $\{0.2, 0.6\}$ kg and velocities of $\{0.2, 0.4\}$ m/s. Notably, the payload configuration of 0.4 kg and velocity of $0.35$ m/s are within the range of the training data, but not explicitly used for training. In contrast, the $0.8$ kg payload and $0.5$ m/s velocity represent extrapolated scenarios that fall outside the training distribution.

\begin{table}[!h]
\footnotesize
\renewcommand{\arraystretch}{1.3}
\caption{\small Performance comparison in real world}
\centering
{
\scalebox{0.75}{
\begin{tabular}{|c||c|c||c|c||c|c|c|c|}
\hline
\multirow{2}{*}{} & \multicolumn{2}{c||}{Payload = $0.4$ kg }                                   & \multicolumn{2}{c||}{Payload = $0.8$ kg } & \multicolumn{1}{c|}{External }                    \\ \cline{2-5} 
\textbf{RMSE} (m) &  $v = 0.35$ m/s        &  $v = 0.5$ m/s         &  $v = 0.35$ m/s         &  $v = 0.5$ m/s   &  Wind       \\ \hline
 $\mathcal{L}_1$-Adaptive & $0.088 $ & $0.101$ & $0.105$ & $0.112 $ & $0.198 $   \\ \hline
GP-AC     & $0.086$ & $0.098$ & $0.102$ & $0.110$ & $0.192$  \\ \hline
DNN-AC    & $0.082$ & $0.095$ & $0.098$ & $0.107$ & $0.184$  \\ \hline
GPT-AC    & $0.075$ & $0.086$ & $0.089$ & $0.102$  & $0.177$  \\ \hline
DM-AC  & $\mathbf{0.069}$ & $\mathbf{0.074}$ & $\mathbf{0.077}$ & $\mathbf{0.089}$  & $\mathbf{0.162}$  \\ \hline
\multicolumn{5}{c}{*\{$\cdot$\}-AC refers to \{$\cdot$\} based adaptive control.}  
\end{tabular}
}}
\label{tab:adaption_payload_wind}
\vspace{-2mm}
\end{table}

Table \ref{tab:adaption_payload_wind} presents a comparison of the tracking performance across different model-based methods. The results clearly show that the model-based framework, which integrates adaptive control with a learned dynamic model, significantly outperforms the $\mathcal{L}_1$-adaptive controller, particularly in scenarios with varying payloads and external wind disturbances. The combined approach enables effective short-term adjustments and long-term compensation for dynamic variations, ensuring both robustness and accuracy. Additionally, the GP and DNN-based methods are less effective at handling the real-world multimodal uncertainties present in these scenarios, resulting in reduced tracking accuracy. While GPT achieves reasonable performance, its autoregressive nature introduces cascading errors, leading to slightly higher tracking errors. In contrast, DM-AC's ability to generate accurate multistep predictions enhances its performance, demonstrating resilience to variations in dynamic parameters and external disturbances.

\subsection{Runtime performance with multi-step predictions}
Sampling from diffusion models generally entails considerable computational expense due to the iterative nature of the denoising process during inference. In our framework, $\mathcal{H}$ is predicted over a finite time horizon $H$ (horizon length), which permits the sequential use of multistep predictions to potentially decrease the frequency of sampling from the diffusion model. However, employing predictions across multiple timesteps to calculate the control inputs $u$ may adversely affect tracking performance. To investigate the impact of computational budget on performance, we tasked the quadrotor to stabilize at a fixed coordinate of $(0.0,0.0,0.8)$ m for $60$ s using DM-AC while increasing the number of predictions used from a sample. Figure \ref{fig:horizon} illustrates the trade-off between performance and Execution Time per time Elapsed over the prediction horizon (ETE). The ETE is defined as the ratio of the time spent evaluating the diffusion model to the time elapsed between successive queries. Our findings reveal that stabilization accuracy is maintained with no crashes when the number of predictions used from a sample is low. However, as the number of predictions is increased to 32 and 64, stability deteriorates resulting in crash rates of $20\%$ and $60\%$, respectively. This degradation in performance can be attributed to the reliance on outdated observations of states and control inputs for estimating $\mathcal{H}$.
\begin{figure}[!h]
\centering
\includegraphics[width=0.97\linewidth]{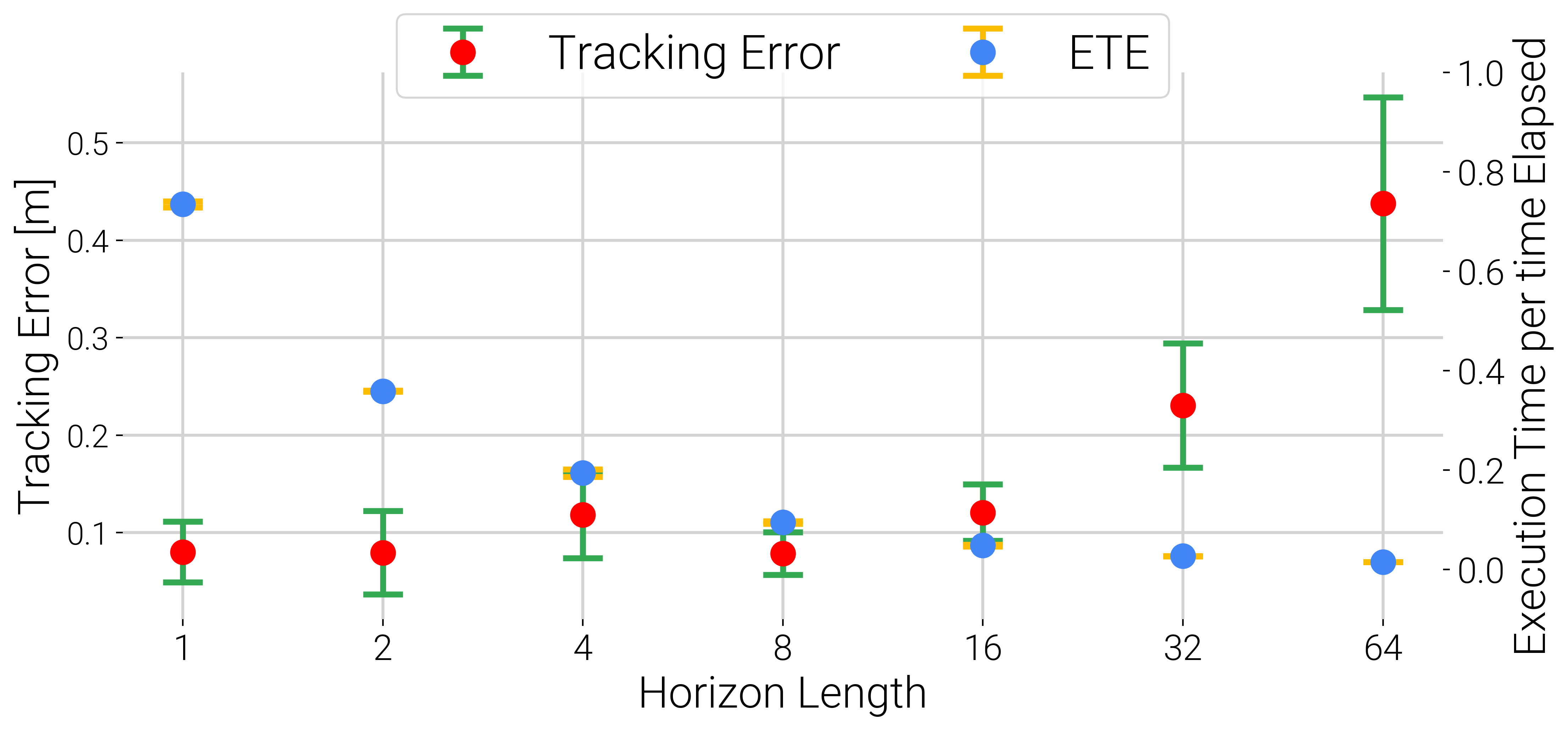}
    \caption{ Hovering accuracy evaluation and execution time for $60$ seconds with different horizon lengths $H$.}
    \label{fig:horizon}
\vspace{-4mm}
\end{figure}

\section{Conclusion and future work}
\label{sec:conclusion}
In this work, we introduced \emph{DroneDiffusion}, a novel framework that effectively leverages diffusion models to capture the stochastic and multimodal nature of quadrotor dynamics. Through training on primitive demonstrations, \emph{DroneDiffusion} exhibits strong generalization capabilities, accurately tracking complex flight paths, including scenarios not encountered during training. The framework's robustness was further validated on real hardware, demonstrating its adaptability to a wide range of payload and velocity variations, as well as its resilience against wind disturbances. These findings establish \emph{DroneDiffusion} as a promising solution for real-world deployment in dynamic and uncertain environments.  Future work can leverage recent advances in efficient sampling techniques~\cite{zheng2024dpm, song2023consistency} to enable onboard inference. Additionally, the potential for extending the framework to agile maneuvering and attitude control is a compelling direction.

\newpage
\appendix

\subsection{Multimodal Nature of $\mathcal{H}$}\label{app:multimodal_appendix}
Figures \ref{fig:multimodal_1} and \ref{fig:multimodal_2} show the empirical density of $\mathcal{H}$ for some additional flight path segments (marked by a {\color{red}red} box). For a quantitative analysis of the nature of $\mathcal{H}$, we divide the flight path (shown in the upper left of Figure \ref{fig:multimodal_1}) into $100$ segments and conduct the \emph{Hartigan's Dip-test for Multimodality}~\cite{hartigan1985dip} on each dimension of $\mathcal{H}$. The p-values along the first, second, and third dimension of $\mathcal{H}$ observed in the {\color{red}red-box} in Figure \ref{fig:multimodal} are $0.01095$, $0.01335$ and $0.00502$ respectively, indicating that even along individual dimensions, $\mathcal{H}$ exhibits multimodal nature with a significance level of $95\%$. To combine the p-values of all $100$ segments of the flight path, we use Fisher's chi-square method~\cite{fisher1970statistical}, Stouffer's method~\cite{stouffer1949american} and Tippett's method~\cite{lehmann1986testing} and report in Table \ref{tab:p_value}. The p-values indicate the multimodal nature of $\mathcal{H}$ in real-world flights.
\begin{figure}[h]
    \centering
    \includegraphics[width=0.95\linewidth]{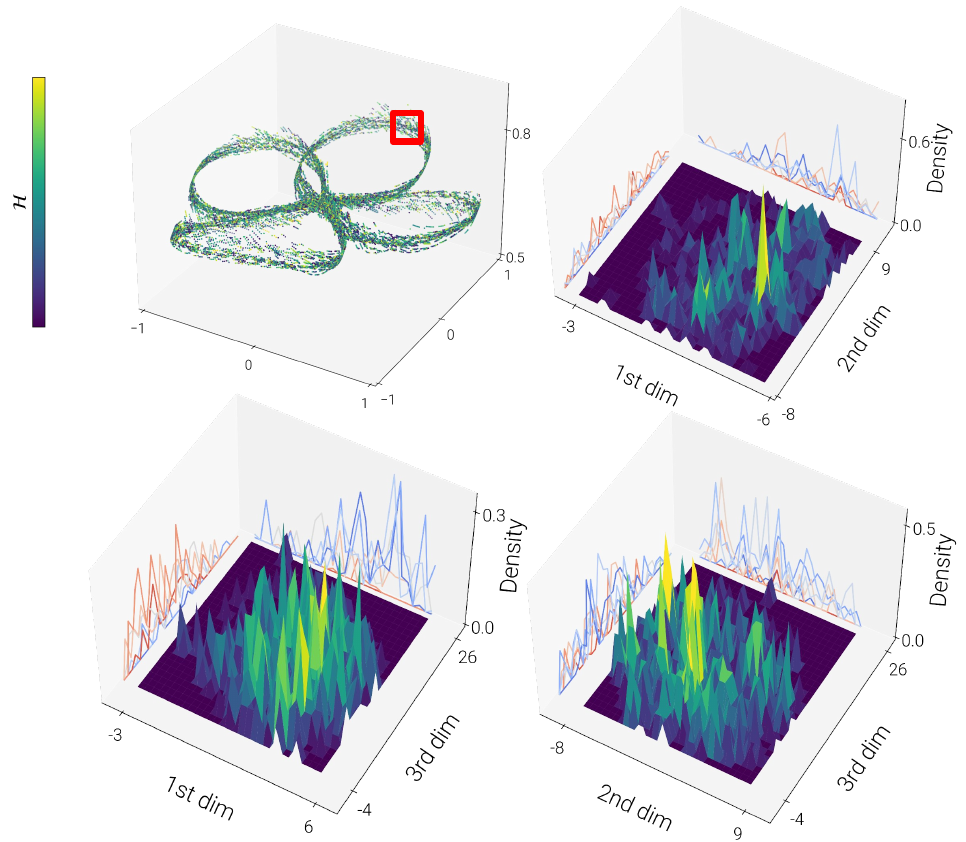}
    \caption{\small {\textit{Plots illustrating the empirical density of $\mathcal{H}$ observed in the {\color{red}red box}(\textit{Top left}) along two dimensions. \textit{Top right:} Density along first and second dimension, \textit{Bottom left:} Density along first and third dimension, and \textit{Bottom right:} Density along second and third dimension.
    }}}
    \label{fig:multimodal_1}
\end{figure}
\begin{figure}[!h]
    \centering
    \includegraphics[width=0.95\linewidth]{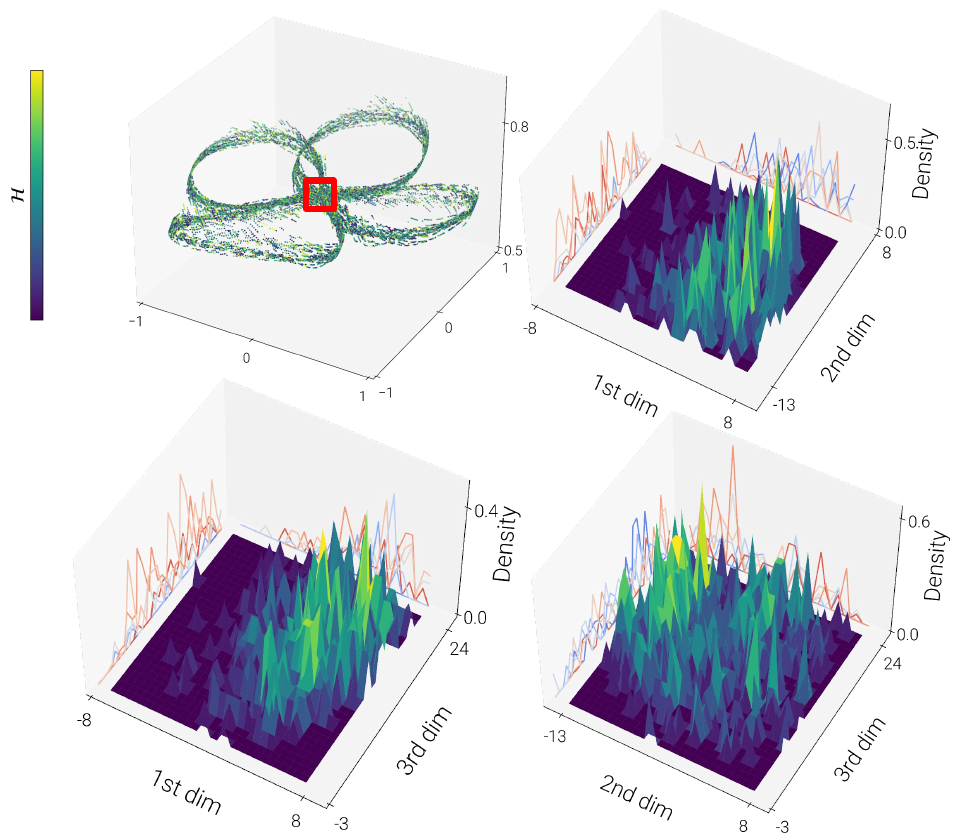}
    \caption{\small {\textit{Plots illustrating the empirical density of $\mathcal{H}$ observed in the {\color{red}red box}(\textit{Top left}) along two dimensions. \textit{Top right:} Density along first and second dimension, \textit{Bottom left:} Density along first and third dimension, and \textit{Bottom right:} Density along second and third dimension.
    }}}
    \label{fig:multimodal_2}
\end{figure}

\begin{table}[!h]
\renewcommand{\arraystretch}{1.2}
\caption{\small {Combined p-values of Hartigan's dip statistic using different methods.}}
		\centering
{
\scalebox{0.95}{	\begin{tabular}{c|c|c|c}
\hline
& {Dimension 1} & {Dimension 2} & {Dimension 3} \\ \hline
{Fisher's $\chi$-square} & {0.0001}      & {0.0000}      & {0.0000}      \\ \hline
{Stouffer's}             & {0.0007}      & {0.0001}      & {0.0275}      \\ \hline
{Tippett's}              & {0.0000}      & {0.0000}      & {0.0004}      \\ \hline
\end{tabular}}}
\label{tab:p_value}
\end{table}

\subsection{Closed-loop Stability Analysis for DM-AC} \label{adaptive_control_design}
Multiplying the time derivative of sliding variable $s$ in  (\ref{sliding_var}) by $\bar{m}$ and using quadrotor dynamics (\ref{eq.p_tau_rearrage}) yields
\begin{equation}
    \bar{m}\dot{s} =  \bar{m}(\ddot{p} - \ddot{p}_d + \Phi \dot{e}_p) = u - \mathcal{H} - \bar{m}(\ddot{p}_d - \Phi \dot{e}_p)
\label{dot_s}
\end{equation}
Using control law $u$ from \eqref{control_law_adaptive_main} in \eqref{dot_s} yields
\begin{align}
    \bar{m}\dot{s} &= -\Lambda s + \hat{\mathcal{H}} - \mathcal{H}  -  \hat{\sigma}(t) (s/ \norm{s}) \\ 
    &= -\Lambda s  + \sigma -  \hat{\sigma}(t) (s/ \norm{s}).
\label{ms_dot}
\end{align}
Here $\sigma$ represents the overall uncertainty in the system resulting from the estimation error ($\hat{\mathcal{H}} - \mathcal{H}$). 

\begin{proof}
    The closed-loop stability analysis is carried out using the following Lyapunov function
\begin{equation}\label{lyp}
    \mathcal{V} =  \frac{1}{2}s^{\top} \bar{m} s + \frac{1}{2} (\hat{\sigma} - \sigma_m)^2.
\end{equation}
Taking the time derivative of $\mathcal{V} $ and using \eqref{ms_dot}, we obtain:
\begin{align}
\dot{\mathcal{V}} &= s^{\top} \bar{m}\dot{s} + (\hat{\sigma} - \sigma_m)\dot{\hat{\sigma}} \nonumber \\
&=  s^{\top}(- \Lambda s  + \sigma -  \hat{\sigma} (s/ \norm{s}))  + (\hat{\sigma} - \sigma_m)\dot{\hat{\sigma}} \nonumber \\
&\leq - \lambda_{\min}(\Lambda) \norm{s}^{2} + \norm{s} \norm{\sigma} - \hat{\sigma} \norm{s} + (\hat{\sigma} - \sigma_m)\dot{\hat{\sigma}}.
\label{v_dot}
\end{align}
The adaptive law (\ref{adaptive_law_main}) yields
\begin{equation}
(\hat{\sigma} - \sigma_m) \dot{\hat{\sigma}} = \norm{s}(\hat{\sigma} - \sigma_m) + \nu \hat{\sigma} \sigma_m -\nu \hat{\sigma}^2.
\label{adaptive_dot}
\end{equation}

Using \eqref{v_dot}, \eqref{adaptive_dot} and utilizing Assumption \ref{assum2}, the upper bound for $\sigma$ is given by $
    \norm{\sigma} \leq \sigma_m$. Therefore, we have: 
\begin{align}
\dot{\mathcal{V}} 
&\leq - \lambda_{\min}(\Lambda) \norm{s}^{2} - (\hat{\sigma} - \sigma_m)(\norm{s}  - \dot{\hat{\sigma}} ) \nonumber \\
&= - \lambda_{\min}(\Lambda) \norm{s}^{2} + (\nu \hat{\sigma} \sigma_m -\nu \hat{\sigma}^2 ) \nonumber \\
&\leq - \lambda_{\min}(\Lambda) \norm{s}^{2} - \frac{\nu}{2}  ((\hat{\sigma} - \sigma_m)^2 - {\sigma_m}^2 ). \label{new2}
\end{align}
Further, the definition of Lyapunov function yields
\begin{equation}
\mathcal{V} \leq \frac{1}{2} \bar{m}||s||^2 + \frac{1}{2} (\hat{\sigma} - \sigma_m)^2. \label{new}
\end{equation}
Substituting (\ref{new}) into (\ref{new2}), $\dot{\mathcal{V}}$ is simplified to
\begin{equation}
\dot{\mathcal{V}} \leq - \varrho \mathcal{V} + \delta,  \label{new3}
\end{equation}
where $\varrho \triangleq  \frac{\min\{ \lambda_{\min}(\Lambda),{\nu}/{2} \}} { \max\{(\bar{m}/2), (1/2)\}} >0$ and $\delta = \frac{1}{2}(\nu {\sigma_m}^2)$. Defining a scalar $\kappa$ such that $0<\kappa<\varrho $, $\dot{\mathcal{V}}$ in (\ref{new3}) yields
\begin{align}
\dot{\mathcal{V}} & \leq -\kappa \mathcal{V} - (\varrho - \kappa)\mathcal{V} + \delta.
\end{align}
Defining a scalar $  \mathcal{B} = \frac{\delta}{(\varrho - \kappa)} $, it can be noticed that $\dot{\mathcal{V}} (t) < - \kappa \mathcal{V} (t)$ when $\mathcal{V} (t) \geq \mathcal{ B}$, so that
\begin{align} \label{ball}
    \mathcal{V} & \leq \max \{ \mathcal{V}(0), \mathcal{B} \}, \forall t \geq 0,
\end{align}
and the Lyapunov function enters in finite time inside the ball defined by $\mathcal{B}$  and the closed-loop system remains UUB.
\end{proof}

\begin{remark} 
For continuity in control law, the term $s/||s||$ in (\ref{control_law_adaptive_main}) is usually replaced by a smooth function $\frac{s}{\sqrt{\norm{s}^2+\varpi}}$ with $\varpi$ being a positive user-defined scalar. This does not alter the overall UUB stability result. $\varpi = 0.1$ has been used to avoid chattering in our experiment.
\end{remark}

\subsection{Diffusion Model-Based Sliding Mode Control (DM-SMC)}\label{sec:dm-smc}
We also propose the Diffusion Model-Based Sliding Mode Control (DM-SMC) for quadrotors as:
\begin{align}
    u(t) &= -\Lambda s (t) + \bar{m} (\ddot{p}_d(t) - \Phi \dot{e}_p(t)) + \hat{\mathcal{H}}, \label{control_law}
\end{align}
Substituting \eqref{control_law} into \eqref{eq.p_tau_rearrage}, the closed-loop dynamics simply becomes $\bar{m}\dot{s} + \Lambda s = \sigma$ with the estimation error $\sigma = \mathcal{H} - \hat{\mathcal{H}}$. 
 As long as $\sigma$ is bounded, $e_p(t)$ will be bounded and converge to a value determined by the magnitude of $\norm {\sigma}$; if $\sigma$ is very small, $e_p(t)$ will converge to a very small value close to zero \cite{slotine1991applied}, as shown in Theorem \ref{theorem.stability}.

\begin{theorem} \label{theorem.stability}
    Under Assumptions \ref{assum1} and \ref{assum2}, the closed-loop trajectories of \eqref{eq.p_tau_rearrage} employing the control law \eqref{control_law} achieve exponential convergence of the composite variable $s$ to the error ball $\lim_{t \to \infty} \norm{s(t)} = \sigma_m/ \lambda_{\min}(\Lambda)$, with convergence rate $\lambda_{\min}(\Lambda) / \bar{m}$, where  $\lambda_{\min} (\Lambda)$  and $\lambda_{\min} (\Phi)$  are the minimum eigenvalue of the positive definite matrix $\Lambda$ and $\Phi$ respectively. Furthermore, $e_p(t)$ exponentially converges to the error ball:
\begin{align} \label{e_errror_ball}
    \lim_{t \to \infty} \norm{e_p(t)} = \frac{\sigma_m}{\lambda_{\min}(\Lambda)\lambda_{\min}(\Phi)} ~\text{with rate} ~\lambda_{\min}(\Phi).
\end{align}
\end{theorem}

\begin{proof}
    To analyze the overall closed-loop stability of the system, we use the Lyapunov function $\mathcal{V} = \frac{1}{2}s^{\top} \bar{m} s$. Under Assumption \ref{assum2}, taking the time derivative of $\mathcal{V} (s)$ and applying the control law  \eqref{control_law} to the system described by \eqref{eq.p_tau_rearrage}, we obtain:
\begin{align*}
\dot{\mathcal{V}} &= s^{\top} \bar{m}\dot{s} =  s^{\top} \bar{m}(\ddot{p} - \ddot{p}_d + \Phi \dot{e}_p) \\
&= s^{\top} (u - \mathcal{H} - \bar{m} \ddot{p}_d + \Phi \dot{e}_p)  \\
&= s^{\top} (-\Lambda s + \hat{\mathcal{H}} - \mathcal{H}) \\
&\leq - \lambda_{\min}(\Lambda) \norm{s}^{2} + \norm{s} \sigma_m.
\end{align*}
where $\lambda_{\min} (\Lambda)$  denotes the minimum eigenvalue of
the positive-definite matrix $\Lambda$. Furthermore,
\begin{align*}
\dot{\mathcal{V}} \leq - \frac{2 \lambda_{\min} (\Lambda)}{\bar{m}} \mathcal{V} + \sqrt{\frac{2\mathcal{V}}{\bar{m}}} \sigma_m
\end{align*}
Using the Comparison Lemma \cite{khalil2002control}, we define $\mathcal{W}(t) = \sqrt{\mathcal{V}(t)} = \sqrt{\bar{m}/2} \norm{s}$ and take its time derivative, $\dot{\mathcal{W}} = \dot{\mathcal{V}}/(2\sqrt{\mathcal{V}})$. This leads to the following inequality:
\begin{align*}
\norm{s(t)} \leq \norm{s(t_0)} \text{exp} \left(- \frac{\lambda_{\min}(\Lambda) }{\bar{m}}(t - t_0) \right) + \frac{\sigma_m}{\lambda_{\min}(\Lambda) }
\end{align*}

It can be shown that this leads to finite-gain $\mathcal{L}_p$ stability and input-to-state stability (ISS) \cite{chung2013phase}. This result is achieved since the hierarchical combination of $s$ and $e_p(t)$ in \eqref{sliding_var} leads to:
\begin{align*}
\lim_{{t \to \infty}} \norm{e(t)} = \lim_{{t \to \infty}} \frac{\norm{s(t)}} {\lambda_{\min}(\Phi)},
\end{align*}
which in turn yields \eqref{e_errror_ball}.
\end{proof}

\begin{remark} [Choice of gains and trade-off] \label{remark.gains_choice} The parameter $\bar{m}$ represents the nominal mass of the quadrotor, and if the exact mass $m$ is known, $\bar{m}$ can be set to $m$. Increasing the gains $\Phi$ results in faster convergence of the position error $e_p(t)$, as indicated by equation \eqref{e_errror_ball}, and a high $\Lambda$ value can reduce the error ball close to zero, due to stronger corrective actions. However, as shown in control laws  \eqref{control_law_adaptive_main} and \eqref{control_law}, very high values of $\Lambda$ and $\Phi$ increase control input demand, potentially leading to actuator saturation and instability if the gains are excessively high.

Further, in the adaptive law (\ref{adaptive_law_main}), the first term controls the growth/adaptation rate, while the negative second term provides stabilization against high gain instability, as defined in \cite[Ch. 8.3]{ioannou1996robust}. Low values of $\nu$ in (\ref{adaptive_law_main}) facilitate faster adaptation with higher adaptive gains, but this comes at the expense of increased control input demand. Conversely, high values of $\nu$  enlarge the ball $\mathcal{B}$ defined in \eqref{ball}, which is undesirable.

Therefore, the choice of the gains $\bar{m}, \Phi, \Lambda, \nu$ should consider application requirements and control input demands. 
    
\end{remark}

\textbf{Diffusion Model based Robust Control (DM-RC): } \label{robust_control_design}
If the upper bound on the estimation error is known (i.e., $\sigma_m$), then a Diffusion Model-based Robust Control (DM-RC) strategy could be implemented as:
\begin{align}
     u(t) &=   -\Lambda s + \bar{m} (\ddot{p}_d - \Phi \dot{e}_p) + \hat{\mathcal{H}} -  {\sigma}_m (s/ \norm{s}).   \label{}
\end{align}
Lower value of the robust gain ${\sigma}_m$ would reduce robustness against uncertainties, while a higher value might be overly conservative, leading to chattering or overcompensation.

\bibliographystyle{IEEEtran}
\bibliography{our_bib}

\end{document}